\documentclass[conference]{IEEEtran}
\usepackage{times}

\usepackage[numbers]{natbib}
\usepackage{multicol}
\usepackage[bookmarks=true]{hyperref}
\usepackage{graphicx, mathtools}

\usepackage{amsmath,amssymb,amsfonts}
\usepackage{algorithm}
\usepackage{algpseudocode}

\usepackage{array}
\usepackage[caption=false,font=normalsize,labelfont=sf,textfont=sf]{subfig}
\usepackage{textcomp}
\usepackage{stfloats}
\usepackage{url}
\usepackage{verbatim}
\usepackage{graphicx}
\usepackage{mathtools}
\usepackage{float}
\usepackage{xcolor}
\usepackage{makecell}
\usepackage{booktabs}  
\usepackage{array, caption, capt-of, cuted}     
\usepackage[font=small,labelfont=bf]{caption}
\usepackage{amsthm}

\let\labelindent\relax
\usepackage{enumitem}

\usepackage{cancel}

\usepackage{svg}
\usepackage{tabularx}

\usepackage{multirow}
\usepackage{makecell}
\usepackage{siunitx}
\usepackage{xcolor}

\newtheorem{theorem}{Theorem}
\newtheorem{lemma}{Lemma}

\newtheorem{assumption}{Assumption}

\usepackage{courier}

\usepackage{eso-pic}

\IEEEoverridecommandlockouts

\begin{document}

\AddToShipoutPictureFG*{%
  \AtPageUpperLeft{%
    \hspace*{0.75in}%
    \raisebox{-0.5in}{%
      {\fontfamily{phv}\fontsize{10}{12}\selectfont\bfseries
      \begin{tabular}{@{}l@{}}
      Robotics: Science and Systems 2026\\
      Sydney, Australia, July 13-July 17, 2026
      \end{tabular}}%
    }%
  }%
}

\title{Distributionally Robust Control via Stein Variational Inference for Contact-Rich Manipulation}

\author{
\authorblockN{Hrishikesh Sathyanarayan\authorrefmark{1},
Victor Vantilborgh\authorrefmark{2},
Harish Ravichandar\authorrefmark{3},
Tom Lefebvre\authorrefmark{2}, and
Ian Abraham\authorrefmark{1}\authorrefmark{4}}
\authorblockA{\authorrefmark{1}Yale University, \authorrefmark{2} Ghent University, \authorrefmark{3}Georgia Tech, \authorrefmark{4}University of Sydney}
\thanks{\authorrefmark{1}H.S and I.A are with the Department of Mechanical Engineering, Yale University, New Haven, CT, USA, email: hrishi.sathyanarayan@yale.edu}
\thanks{\authorrefmark{2}V.V and T.L are with the Department of Electromechanical, Systems and Metal Engineering, Ghent University, Belgium, email: \{ victor.vantilborgh, tom.lefebvre\}@ugent.be}
\thanks{\authorrefmark{3}H.R. is with the School of Interactive Computing, Georgia Institute of Technology, Atlanta, GA, USA, email: harish.ravichandar@cc.gatech.edu}
\thanks{\authorrefmark{4}I.A is with the Department of Electrical Engineering, University of Sydney, Australia, email: ian.abraham@sydney.edu.au}
}



%

\maketitle

\begin{abstract}
Reliable robotic manipulation requires control policies that can accurately represent and adapt to uncertainty arising from contact-rich interactions. 
Modern data-driven methods mitigate uncertainty through large-scale training and computation, and degrade significantly in performance with limited number of training samples. 
By contrast, classical model-based controllers are computationally efficient and reliable, but their limited ability to represent task-relevant uncertainty can hinder performance in contact-rich interactions. 

In this work, we propose to expand the capabilities of model-based manipulation control through more flexible uncertainty modeling that retains performance while exactly adapting to uncertainty. 
Our approach casts the manipulation problem as a distributionally robust control optimization and proposes a novel deterministic formulation based on Stein variational inference that preserves performance while explicitly modeling task-sensitive parameter uncertainty. 
As a result, the derived controllers are more aware of task sensitivities to uncertainty, yielding high reliability without compromising performance.
Experimental results demonstrate up to 3$\times$ improved robustness across a range of contact-rich manipulation tasks under broad parametric uncertainty, outperforming existing model-based control methods. Additional media and code is provided in \url{https://github.com/ialab-yale/stein-variational-dro.git}


\end{abstract}

\IEEEpeerreviewmaketitle

\section{Introduction}

   In-the-wild manipulation often requires reasoning in an environment filled with uncertainty.
    However, the reliability and success of robotic control for manipulation is contingent on the ability to appropriately exploit and adapt to uncertainty through interacting with their environment.
    Modern control approaches mitigate uncertainty through vast accumulation of data and compute  \cite{fiedler2021learningenhancedrobustcontrollersynthesis,trilbmteam2025carefulexaminationlargebehavior,chi2024diffusionpolicyvisuomotorpolicy,bommasani2022opportunitiesrisksfoundationmodels,embodimentcollaboration2025openxembodimentroboticlearning}, effectively aiming to cover all possible scenarios at the expense of precision and interpretability, and their performance rapidly degrades with a reducing number of samples \cite{levine2016endtoendtrainingdeepvisuomotor,dulacarnold2021empiricalinvestigationchallengesrealworld}. 
    The unifying goal is to develop zero-shot manipulation controllers that actively adapt to physical uncertainty, without the loss of performance and robustness.

    \begin{figure}
        \centering
        \includegraphics[width=\linewidth]{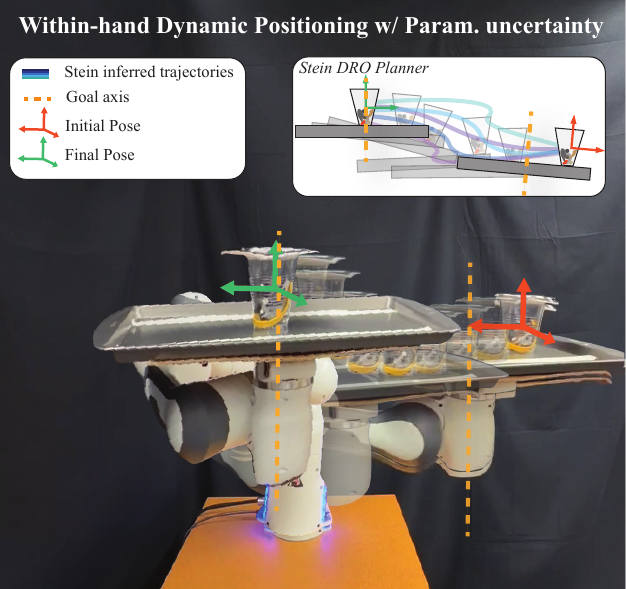}
        \caption{\textbf{Within-hand dynamic positioning of a cup with unknown mass distribution and friction coefficient. } Our proposed Stein Variational Distributionally Robust Optimizer (SV-DRO) applied to a within-hand positioning task via controlled sliding.
        The demonstration shown requires the robot to slide an object with unknown physical parameters (inertia, mass, and friction coefficient) to a goal state located at the center of the tray.
        Our approach enables control policies to reason over uncertainty-sensitivity to task ensuring reliable performance.}
        \label{fig:abstract}
        \vspace{-2em}
    \end{figure}

    Most prior work on uncertainty-aware manipulation control under parameter uncertainty spans data-driven methods and classical adaptive control approaches, each expressing a different facet on robust optimization \cite{ZHANG2017188,https://doi.org/10.1002/rnc.1823,doi:10.1126/scirobotics.adt1497}. 
    Reinforcement learning (RL) and related policy optimization methods in \cite{heess2015learningcontinuouscontrolpolicies,openai2019learningdexterousinhandmanipulation} have been extensively applied to contact-rich tasks, where the ability to implicitly learn from experience offers potential adaptability in uncertainty-filled environments. 
    These methods, however, often necessitate high compute and data quality, and sufficient coverage over all possible parameter uncertainties for high-performance guarantees \cite{tobin2017domainrandomizationtransferringdeep,Peng_2018,rajeswaran2017epoptlearningrobustneural}.
    Other works presented in \cite{Abraham_2020,lowrey2019planonlinelearnoffline,7354126} leverage sampled trajectory ensembles to mitigate the effects of parametric uncertainty through optimizing controls over average performance.
    These sample-based ensemble approaches have limitations, however, in that they fail to adapt robustly to parameter uncertainty under a prohibited number of samples, and compute only \emph{average} performance that is highly dependent on the range of parameter uncertainty. 
    Other, classical control methods \cite{liu2025datadrivendistributionallyrobustoptimal,long2025sensorbaseddistributionallyrobustcontrol,8593893,10161246} promise control policies from conservative, worst-case formulations that adapt to uncertainty via a Distributionally Robust Optimization (DRO) formulation \cite{kuhn2025distributionallyrobustoptimization}.
    These methods compromise task performance in response to worst-case uncertainty modeling leading to worse overall performance in contact-rich settings due to the variability of object dynamics subject to contact.
    
    \emph{Thus, the goal of this work is to expand the capabilities of model-based control for robotic manipulation through more flexible uncertainty modeling that retains performance while adapting to uncertainty in manipulation. }

    To achieve this goal, we formulate robotic control under parametric uncertainty as a Distributionally Robust Optimization problem and introduce a Stein variational inference-based controller~\cite{steinOG,liu2019steinvariationalgradientdescent}. 
    Rather than optimizing against a fixed ambiguity set, we use Stein variational gradient descent (SVGD) to evolve deterministic particles drawn from a prior into a task-aware parameter posterior. 
    This provides a non-parametric, parallelizable approximation of how parameter uncertainty affects task performance~\cite{YAN2021114087,Li_2020}. 
    Fig.~\ref{fig:abstract} illustrates this idea on within-hand manipulation, where an object with uncertain mass distribution and friction is reliably moved to the tray center. 
    Overall, SV-DRO uses exact\footnote{SVGD can exactly match expectations of functions within a Stein-induced function class and converges to the target distribution in the infinite-particle limit~\cite{liu2019steinvariationalgradientdescent}; with finite particles, it gives an approximate posterior representation.} Stein inference to identify uncertainty-sensitive actions critical to contact-rich manipulation.
    
    In practice, SV-DRO achieves strong task performance under broad parameter uncertainty and limited feedback compared with existing model-based baselines~\cite{Abraham_2020,kuhn2025distributionallyrobustoptimization}. 
    The central distinction is philosophical: instead of planning for a globally worst-case model, SV-DRO reshapes uncertainty toward parameter variations that most affect the current task and optimizes against those task-sensitive particles. 
    We demonstrate this behavior on contact-rich manipulation tasks including within-hand dynamic positioning and bimanual Push-T.
    
    \emph{To summarize, the contributions of this paper are the following:}
    \begin{itemize}
        \item Derivation of a Stein variational inference-based distributionally robust controller that accurately models uncertainty and generates control solutions that enhances task performance.
        \item Demonstration of 3$\times$ improvement of our proposed method compared against existing DRO and ensemble worst-case manipulation controllers subject wide ranging parameter uncertainty. 
        \item Showcasing continually reliable within-hand manipulation on a rolling experimental setting.
    \end{itemize}

    The remainder of this paper is structured as follows: Section~\ref{sec:related_work} discusses prior relevant work. Section~\ref{sec:preliminaries} 
    overviews preliminary material on Distributionally Robust Control and Stein Variational Gradient Descent. Section \ref{sec:main} outlines our proposed method that combines Stein Variational Inference with DRO, and details theoretical analyses on our method. And finally, Section \ref{sec:results} discusses experimental results, Section \ref{sec:limitation} covers limitations, and Section \ref{sec:conclusion} proposes  future research directions.
    
\section{Related Work}
\label{sec:related_work}

    \subsection{Distributionally Robust Optimization}

        Distributionally Robust Optimization (DRO) optimizes performance against worst-case probability distributions within an ambiguity set, providing robustness to uncertainty at the cost of potential conservatism~\cite{kuhn2025distributionallyrobustoptimization,Rahimian_2022}. 
        Prior work has made DRO tractable for control by using duality under Kullback--Leibler ambiguity sets, reducing the original min--max problem to a single-level optimization~\cite{liu2025datadrivendistributionallyrobustoptimal,8593893}. 
        However, these approaches typically require a parametric reference distribution estimated through maximum-likelihood or related data-driven procedures, so robustness is only guaranteed near a learned nominal uncertainty model. 
        Wasserstein DRO formulations avoid explicit distribution fitting by using empirical uncertainty samples~\cite{long2025sensorbaseddistributionallyrobustcontrol}, but they assume direct access to samples from the uncertain process~\cite[Assumption 6.3]{long2025sensorbaseddistributionallyrobustcontrol,doi:10.2514/6.2018-0666}, which is restrictive when uncertainty corresponds to latent physical parameters such as mass, inertia, friction, or geometry.
        
        These assumptions are especially limiting in contact-rich manipulation, where latent parameters cannot usually be sampled directly, feedback may be sparse or indirect, and contact dynamics introduce discontinuous, mode-dependent behavior outside the smooth settings considered by many DRO methods~\cite{liu2025datadrivendistributionallyrobustoptimal,long2025sensorbaseddistributionallyrobustcontrol}. 
        Moreover, optimizing against all admissible perturbations can produce overly conservative behavior that sacrifices task performance. 
        In this work, we extend DRO-based control by using Stein variational inference to evolve particles over latent physical parameters according to task-relevant sensitivity, rather than relying on a fixed reference distribution or a sampling oracle. This enables robust MPC for contact-dominated manipulation while focusing computation on uncertainty that directly affects task success.

    \subsection{Stein Variational Inference in Robot Control}

        Variational Inference (VI) approximates an intractable target distribution by optimizing a surrogate distribution to minimize the Kullback-Leibler divergence, but its effectiveness depends on choosing an expressive variational family \cite{steinOG}, which is difficult in high-dimensional or multi-modal settings. 
        Stein variational inference, or Stein variational gradient descent (SVGD), avoids this restriction by evolving a set of particles through deterministic functional gradient descent in a reproducing kernel Hilbert space \cite{liu2019steinvariationalgradientdescent}. 
        This particle-based update combines attraction toward high-probability regions with a kernel-induced repulsion term that preserves diversity, making SVGD well suited for representing complex uncertainty distributions in control. 
        Prior work has used SVI in MPC to maintain multi-modal control particles \cite{lambert2021steinvariationalmodelpredictive}, optimize constrained trajectory distributions \cite{power2024constrainedsteinvariationaltrajectory}, generate informative failures \cite{huang2025fail2progresslearningrealworldrobot}, perform task and motion planning \cite{lee2024stampdifferentiabletaskmotion}, and jointly infer controls and dynamics parameters online \cite{barcelos2021dualonlinesteinvariational}.
        
        Our use of SVI differs from estimation-driven Stein MPC methods \cite{barcelos2021dualonlinesteinvariational} that update parameter particles from online state-transition measurements to recover a dynamics posterior \cite{barcelos2021dualonlinesteinvariational}. Such approaches require sufficiently informative feedback, which can be sparse, low-frequency, noisy, or only indirectly related to latent contact parameters such as mass distribution, inertia, and friction in contact-rich manipulation. In contrast, we do not use SVGD to identify the true physical parameters through a measurement-likelihood, contact-aware maximum a-posteriori, or other contact-aware active learning objective \cite{sathyanarayan2025behaviorsynthesiscontactawarefisher, vantilborgh2025dualcontrolreferencegeneration, doi:10.1073/pnas.2011905118}. 
        Instead, we construct a task-aware parameter posterior induced by the optimality gap, transport particles toward regions where the planned trajectory is most sensitive to uncertainty, and embed the resulting particle set inside a distributionally robust MPC objective. This allows the controller to optimize against uncertainty that directly affects task success, rather than uncertainty that merely best explains recent state feedback.

    \subsection{Domain Randomization}

        Methods in reinforcement learning (RL) \cite{zhang2020sampleefficientreinforcementlearning,NIPS1999_6449f44a} based on domain randomization \cite{8460528,8202133} aim to achieve robustness under parametric uncertainty by training policies over a distribution of system parameters sampled offline. 
        In this setting, a policy is optimized to perform well in expectation across the randomized domain, implicitly learning adapting to variations in model dynamics or environmental scenery. 
        While effective, such approaches typically require extensive data collection and training, and their performance depends critically on the choice and coverage of the parameter distribution used during training \cite{9308468,tiboni2024domainrandomizationentropymaximization}. 
        Moreover, the resulting domain randomization process is inherently \emph{task-agnostic} with respect to parameters with larger uncertainties, as the prior distribution used for sampling is fixed a priori and not adapted based on the specific control objective at execution time.

        In contrast, our approach performs online, task-aware adaptation of the parameter distribution within a model predictive control (MPC) framework, eliminating the need for offline training or repeated system interaction. 
        By leveraging SVI, the parameter distribution is iteratively reshaped according to the optimality gap, focusing computation on uncertainty directions that directly impact task performance. 
        This enables zero-shot control under uncertainty while maintaining computational tractability, and yields policies that are both robust and dynamically responsive to task-specific sensitivities. 
        As such, our method can be viewed as complementary to domain randomization, offering a more data-efficient alternative that emphasizes task-relevant uncertainty at test time rather than distributional coverage during training.

\section{Preliminaries}
\label{sec:preliminaries}

    \subsection{Distributionally Robust Control}

        We now formalize the control problem under parametric uncertainty using a distributionally robust perspective. 
        In this setting, the objective is to synthesize a trajectory that minimizes a task-dependent cost while remaining robust to parameter uncertainty.
        To do so, we first consider the following optimization problem,
        \begin{equation}\label{eq:task_opt}
            \begin{split}
                &\min_{\tau=\{(x_k,u_k,c_k)\forall k\in[0,t_h]\}} \mathcal{J}_\theta(\tau) \\
                \textrm{s.t}
                &\begin{cases}
                    x_0 \in \mathcal{X} & \textrm{(init. state)} \\
                    x_{k+1} = f_\theta(x_k,u_k) & \textrm{(dynamical model)} \\
                    c_{k+1} \in \mathcal{C}(x_k) & \textrm{(contact model)}
                \end{cases}
            \end{split}
        \end{equation}
        where $\mathcal{J}_\theta(\cdot): \mathcal{X} \times \mathcal{U} \times \mathcal{C} \times\Theta \rightarrow \mathbb{R}$ is the task-relevant cost to minimize, $\tau=\{(x_k,u_k,c_k)\forall k\in[t,t+t_h]\}$ is a sequence of states $x\in\mathcal{X}$, control inputs $u\in\mathcal{U}$, and contact impulses $c_k \in \mathcal{C}$, with $t \in [0,T]$ being the current time iterate within $T$ total time iterations and $t_h$ being the planning horizon, and $\theta\in\Theta$ are physics model parameters.
        The Lagrangian $\mathcal{L}: \mathcal{X} \times \mathcal{U} \times \mathcal{C} \times\Theta\rightarrow \mathbb{R}$ of the above optimization problem can be written as,
        \begin{equation}\label{eq:lagrangian}
            \mathcal{L}(\tau, \theta) = \mathcal{J}_\theta (\tau) + \lambda^\top h_\theta(\tau)
        \end{equation}
        where $\lambda \in \mathbb{R}^{n_{eq}}$ is a vector of Lagrange multipliers, and $h_\theta: \mathcal{X} \times \mathcal{U} \times \mathcal{C} \times\Theta\rightarrow \mathbb{R}^{n_{eq}}$ is the vector of equality constraints from Eq. \ref{eq:task_opt}.
        In this paper, we implement Eq. \ref{eq:task_opt} via forward rollouts under the surrogate model $f_\theta$ (which we construct), which computes contact impulses using a spring-like contact model \cite{leesoft}, and execute the first control input in a receding-horizon manner, avoiding the need to explicitly enforce equality constraints in practice.
        Distributionally Robust Optimization (DRO) seeks to optimize over the Lagrangian via a bi-level optimization problem,
        \begin{equation}
            \min_\tau \sup_{p \in \mathcal{P}} \mathbb{E}_{\theta \sim p(\theta)} \left[ \mathcal{L}(\tau,\theta) \right]
        \end{equation}
        where $p$ is a surrogate distribution over $\theta$, and $\mathcal{P}$ is the following ambiguity set,
        \begin{equation}
            \mathcal{P} = \left\{p \ | \ \mathbb{D}_{KL} \left[ p \| q\right] \leq \epsilon\right\}
        \end{equation}
        where $q$ is true unknown target distribution.
        Since finding an ambiguity set $\mathcal{P}$ is challenging to find a-priori, the above DRO problem is commonly approximated as shown in \cite{liu2025datadrivendistributionallyrobustoptimal} as,
        \begin{equation}
            \label{eq:softdro}
            \min_{\tau, \beta > 0} \beta \epsilon + \beta \log \mathbb{E}_q \Bigg[ \exp\Big(\frac{1}{\beta} \mathcal{L}(\tau, \theta)\Big) \Bigg],
        \end{equation}
        where $\beta$ is a dual auxiliary variable.
        Optimizing the second term is more commonly known as the Risk-Averse Optimal Control Problem \cite{Whittle_1981,Ian1973OptimalSL}.
        Note that a similar objective was derived in~\cite{Abraham_2020} through the ensemble planning approach.
        For sufficiently large $\beta$, it holds true via second order Taylor Expansion that,
        \begin{equation}
            \label{eq:log_expectation_dro_approx}
            \beta \log \mathbb{E}_q \Bigg[ \exp\Big(\frac{1}{\beta} \mathcal{L}(\tau, \theta)\Big) \Bigg] \approx \mathbb{E}_q \big[ \mathcal{L}(\tau, \theta) \big] + \frac{1}{2} \frac{1}{\beta} \mathbb{V}_q \big[ \mathcal{L}(\tau, \theta) \big]
        \end{equation}

        \begin{assumption}\label{assum:epsilon_pq}
            (KL-ball ambiguity set)\cite{liu2025datadrivendistributionallyrobustoptimal,long2025sensorbaseddistributionallyrobustcontrol} the surrogate distribution $p$ exists in the $\epsilon$-neighborhood of the target distribution, where $\mathbb{D}_{KL}(p \| q) \leq \epsilon$ given small $\epsilon > 0$. 
        \end{assumption}
        To ensure that the above assumption holds, \cite{liu2025datadrivendistributionallyrobustoptimal} first estimates the target distribution through a data-driven maximum-likelihood estimation approach prior to k-Nearest-Neighbor sampling from the resulting inferred distribution. 
        The above assumption already holds if terms in Eq. \ref{eq:log_expectation_dro_approx} is approximated as,
        \begin{equation}
            \begin{split}
                \mathbb{E}_q \big[ \mathcal{L}(\tau, \theta) \big] &\approx \mathcal{L}(\tau, \theta) \Big|_{\theta=\bar{\theta}} \\
                \mathbb{V}_q \big[ \mathcal{L}(\tau, \theta) \big] &\approx \nabla_\theta \mathcal{L}(\tau, \theta) \Big|_{\theta=\bar{\theta}}^\top \Sigma_{\theta\theta} \nabla_\theta \mathcal{L}(\tau, \theta) \Big|_{\theta=\bar{\theta}}
            \end{split}
        \end{equation}
        where $\bar{\theta}$ is the parameter mean value, and $\Sigma_{\theta\theta} \in \mathbb{R}^{\Theta \times \Theta}$ is the variance about the parameter.

        Note that Assumption \ref{assum:epsilon_pq} may not always be guaranteed in the general case, and is scarce in a variety of manipulation problems.
        Therefore, implementing standard DRO in a manipulation setting filled with high uncertainty, our priors may not be strong enough to ensure that Assumption \ref{assum:epsilon_pq} holds.
        In this work, we develop a novel approach that extends DRO-based methods without the need of asserting Assumption \ref{assum:epsilon_pq} by exactly measuring uncertainty through a guided task-aware variational inference-based method.

    \subsection{Stein Variational Gradient Descent}

        Given a random variable $\theta \in \Theta$, the goal of variational inference is to minimize the Kullback-Leibler (KL) divergence measure over the target distribution $q(\theta)$ and surrogate $p(\theta)$,
        \begin{equation}
        \begin{split}
            p^*(\theta) &= \arg \min_{p\in\mathcal{P}} \{\mathbb{D}_{KL} (p \| q) \\
            &\equiv  \mathbb{E}_{\theta \sim p(\theta)} \big[\log p(\theta)\big] - \mathbb{E}_{\theta \sim p(\theta)} \big[\log q(\theta)\big] + \log z\}
        \end{split}
        \end{equation}
        where $z$ is a constant and thus ignored in the optimization.
        Stein Variational Gradient Descent (SVGD) \cite{liu2019steinvariationalgradientdescent} offers a solution through kernel and gradient descent methods that avoids committing to a parametric variational family $\mathcal{P}$ entirely.

        SVGD initializes a set of points $\{\theta_0^i\}_{i=1}^N$ from a prior distribution $p(\theta)=\mathcal{N}(\hat{\theta},\Sigma_\theta)$ in $\theta$. The points evolve according to the one-to-one mapping,
        \begin{equation}
            \theta_{t+1}^i \gets \theta_t^i + \alpha \ \phi_{p,q}^*(\theta_t^i)
        \end{equation}
        where $\phi_{p,q}^*(\theta_t^i)$ is a smooth function that captures the steepest descent direction that minimizes the KL-divergence measure at the $t^{\textrm{th}}$ step, with $T$ being the total number of SVGD iterations, and $\alpha>0$ is a sufficiently small perturbation magnitude.
        We consider $\phi_{p,q}^*$ to be the solution to the steepest descent problem,
        \begin{equation}\label{eq:steepest_desc_opt}
            \begin{split}
                \phi_{p,q}^*(\cdot) &= \arg \min_{\phi \in \mathcal{H}^d} \{ -\nabla_\theta D_{KL} (p \| q) \ | \ \| \phi \|_{\mathcal{H}^d} \leq 1 \} \\
                &=\mathbb{E}_{\theta_t \sim q} \big[\mathcal{A}_q k(\theta_t,\cdot) \big]
            \end{split}
        \end{equation}
        where $\mathcal{A}_q(\cdot): \Theta \rightarrow \mathcal{S}_\Theta$ is Stein's identity computed for a universal positive definite kernel function $k: \Theta \times \Theta \rightarrow \mathbb{R}$ operating in a dense $\mathcal{H}^d$ in the space of continuous functions $C(\Theta,\mathbb{R}^d)$, where $\mathcal{H}^d$ is the corresponding Reproducing Kernel Hilbert Space (RKHS).
        The closed form solution to Eq. \ref{eq:steepest_desc_opt} is computed to be,
        \begin{equation}
            \begin{split}
                \phi_{p,q}^*(\cdot) = \mathbb{E}_{\theta_t \sim p}\big[k(\theta_t,\cdot) \nabla_\theta \log q(\theta_t) + \nabla_\theta k(\theta_t,\cdot) \big]
            \end{split}
        \end{equation}
        In practice, it is common to approximate the steepest descent direction $\phi_{p,q}^*$ using Monte-Carlo sampling over random variables $\theta$,
        \begin{equation}
            \phi_{p,q}^*(\cdot) \approx \frac{1}{N} \sum_{t=1}^N k(\theta_t^i,\cdot) \nabla_\theta \log q(\theta_t^i) + \nabla_\theta k(\theta_t^i,\cdot).
        \end{equation}
        Over sufficient number of iterations of SVGD and number of samples $N$, the sampled parameters become a sufficient approximation of the target $q(\theta)$.
        In this work, we utilize SVGD as a guiding method to evolve parameter points initially sampled from a prior based on task-aware posterior gradients in order to deterministically model the parameter distribution subject to task-sensitivities and robustly solve contact-rich manipulation tasks.
        Note that while contact-based control problems are non-smooth in nature, we utilize a Lagrangian function that is smooth to take gradients over via SVGD.

    \begin{figure*}
        \centering
        \includegraphics[width=\linewidth]{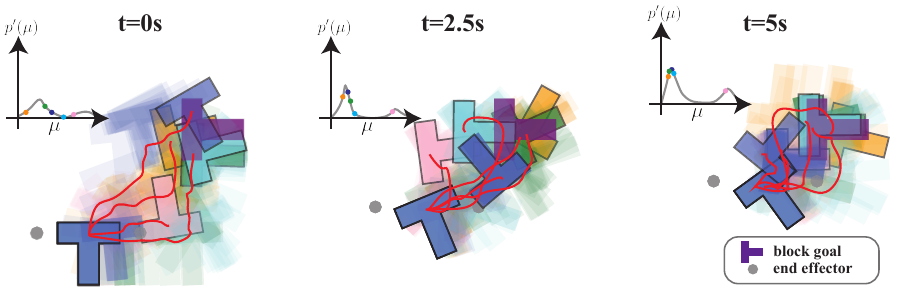}
        \caption{
        \textbf{Illustrative Example of SV-DRO.}
        We visualize SV-DRO on the bimanual Push-T task, where two end effectors move a block to a goal under uncertain block--ground friction $\mu$. 
        Each translucent trajectory is a rollout induced by a different friction particle, initially sampled from the prior and then transported by SVGD using the task optimality gap. 
        Rather than identifying friction before control or optimizing against a fixed worst case, SV-DRO reshapes uncertainty online toward parameter realizations most consequential for task success.
        }
        \label{fig:intuition}
        \vspace{-2em}
    \end{figure*}
\section{Stein Variational Distributionally Robust Control}
\label{sec:main}

    In this section, we derive the main contribution of this work that formulates the Distributionally Robust Control problem via Stein variational inference. 
    \subsection{Variational Approximation of Task-Dependent Parameter Posterior}
        We first formulate the parameter posterior distribution via SVI. 
        Given a set of randomly drawn parameter points $\{\theta\}_{i=1}^N \sim p(\theta)$, where $p(\theta)$ is a prior distribution, we approximate the task-dependent parameter posterior given a likelihood output $\mathcal{O}$ as,
        \begin{equation}
            \begin{split}
                p(\mathcal{O} \mid \theta) = \exp\big(\delta \mathcal{L}(\tau^\ast,\theta) \big)
            \end{split}
        \end{equation}
        where $\tau^\ast$ is the solved optimal trajectory, and the likelihood posterior $p(\mathcal{O} \mid \theta)$ relates to the target $q^\ast(\theta)$ as,
        \begin{equation}
            q^*(\theta) = \arg \min_{q\in Q} D_{KL} \Big(q(\theta) \,||\, p(\mathcal{O}|\theta) \Big).
        \end{equation}
        Here, $\delta \mathcal{L}=\mathcal{L}(\tau^*,\theta)-\mathcal{L}(\tau^*,\theta^*)$ denotes the optimality gap, where $\theta^*$ is the true parameter and $\tau^*$ is the solved trajectory. 
        Since $\theta^*$ is unavailable, we use the particle mean $\theta^* \approx \frac{1}{N}\sum_{i=1}^N \theta^i$ as a stable, low-variance reference for evaluating optimality, which is a standard and consistent estimator that provides stable low-variance reference for defining the optimality gap. 
        The key idea is not to recover the true parameter, but to infer a task-consistent posterior whose particles capture the parameter variations that matter for producing a single robust control action under uncertainty. 
        The likelihood output $\mathcal{O}$ encodes parameter-dependent outcomes in task optimality.
        The posterior is then calculated as a parameter-conditioned optimality gap distribution,
        \begin{equation}
            p'(\theta) = \frac{p(\mathcal{O}|\theta)p(\theta)}{z} = \frac{\exp \big(\delta \mathcal{L}(\tau^\ast,\theta) \big) \ p(\theta)}{z}
        \end{equation}
        where $z=\int_\Theta p(\mathcal{O}|\nu)p(\nu) d\nu$ which reveals the following log posterior,
        \begin{equation}
        \label{eq:log_posterior}
                \log p'(\theta) = \log p(\theta) + \delta \mathcal{L}(\tau, \theta) + \cancelto{\textrm{constant}}{\log z}.
        \end{equation}
        We now seek to utilize the posterior distribution to evolve the $t^\textrm{th}$ sampled set of estimators $\{\theta_t^i\}_{i=1}^N$ drawn from a prior $p(\theta)$ by evaluating the steepest descent direction $\phi^*: \Theta \rightarrow \mathcal{S}_\Theta$ using Stein's identity.
        Given a prior $p(\theta)$ and a universal positive definite kernel $k: \Theta \times \Theta \rightarrow \mathbb{R}$, the resulting SVGD steepest descent solution to Eq. \ref{eq:steepest_desc_opt} is,
        \begin{equation}\label{eq:steepest_desc_params}
            \begin{split}
                \phi^*(\cdot) &= \mathbb{E}_{\theta \sim p(\theta)} \big[ k(\theta^i,\cdot) \nabla_\theta \log p'(\theta^i) + \nabla_\theta k(\theta^i,\cdot) \big] \\
                &\approx \frac{1}{N} \sum_{i=1}^N k(\theta_t^i,\cdot)\nabla_{\theta} \log p'(\theta_t^i) + \nabla_{\theta} k(\theta_t^i,\cdot)
            \end{split}
        \end{equation}
        where $k(\theta,\hat{\theta})=\exp (-\|\theta-\hat{\theta}\|^2/h)$ is chosen as a Radial Basis Function (RBF) kernel~\cite{zhuo2018messagepassingsteinvariational}.
        
        Because the true physics model parameters $\theta^*$ are unknown and inestimable, we require a method that expands on existing DRO formulations  \cite{kuhn2025distributionallyrobustoptimization,liu2025datadrivendistributionallyrobustoptimal,long2025sensorbaseddistributionallyrobustcontrol} by \emph{diversifying} the parametric inferences $\theta \sim p(\theta)$ through an evolving task-aware surrogate posterior that exactly adapts to the uncertainty present.
        The benefit of performing SVGD for parametric inference is that we can compute varied, diverse parameter-sensitive gradients to the task across a wide range $\theta$, thus reasoning over the manipulation task over an \emph{ensemble} of sampled parameter estimates that accurately estimates the task-aware posterior.
        
        By utilizing SVI, we focus \emph{limited computational resources} on task-critical uncertainties. 
        By iteratively shaping the uncertainty distribution toward high-risk regions, we derive a robust controller that avoids the conservatism of worst-case methods while maintaining the sample efficiency of model-based approaches.

    \begin{figure*}[t!]
        \centering
        \includegraphics[width=\linewidth]{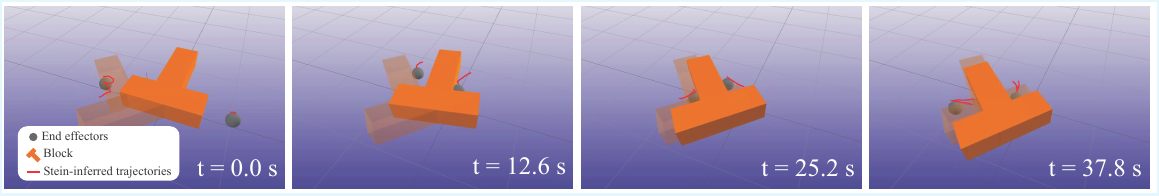}
        \caption{\textbf{Bimanual manipulation of Push-T with unknown mass distribution. } 
        Here, we demonstrate a control solution to the bimanual push-T task using our proposed SV-DRO.
        We observe that despite high uncertainty in mass, the bimanual end-effectors are able to successfully manipulate the T-block to the desired position while avoiding task-sensitive actions as a function of the uncertain parameters.}
        \label{fig:push-T-example-traj}
    \end{figure*}

    \begin{table*}[t]
        \centering
        \small
        \setlength{\tabcolsep}{5pt}
        \renewcommand{\arraystretch}{1.15}
    
        \begin{tabularx}{\textwidth}{@{} l *{4}{>{\centering\arraybackslash}X} @{}}
            \toprule
            & \multicolumn{4}{c}{\textbf{Success \% / Time to Completion (s)}} \\
            \cmidrule(lr){2-5}
            & \multicolumn{2}{c}{\textbf{Bimanual Push-T}}
            & \multicolumn{2}{c}{\textbf{Within-Hand Dynamic Positioning}} \\
            \cmidrule(lr){2-3} \cmidrule(lr){4-5}
    
            \textbf{Method}
            & $\leq 10$ cm
            & $\leq 1$ cm
            & $\leq 10$ cm
            & $\leq 1$ cm \\
            \midrule
    
            \textbf{SV-DRO (Ours)}
            & \bfseries 93.25\% / 16.45 $\pm$ 2.45
            & \bfseries 84.38\% / 14.21 $\pm$ 2.22
            & \bfseries 100\% / 12.00 $\pm$ 4.78
            & \bfseries 100\% / 12.00 $\pm$ 4.78 \\

            DuST-MPC \cite{barcelos2021dualonlinesteinvariational}
            & 71.88\% / 14.41 $\pm$ 3.58
            & 59.38\% / 13.32 $\pm$ 2.67
            & 84.38\% / 20.21 $\pm$ 5.62
            & 46.88\% / 20.18 $\pm$ 1.24 \\
    
            EMPPI \cite{Abraham_2020}
            & 43.75\% / 18.88 $\pm$ 6.29
            & 31.25\% / 14.89 $\pm$ 3.12
            & 59.38\% / 22.44 $\pm$ 3.22
            & 40.63\% / 21.78 $\pm$ 3.43 \\
    
            MPC
            & 28.13\% / 19.56 $\pm$ 6.21
            & 18.75\% / 16.56 $\pm$ 2.21
            & 31.25\% / 18.40 $\pm$ 5.60
            & 25\% / 16.80 $\pm$ 5.14 \\
    
            DRO \cite{kuhn2025distributionallyrobustoptimization}
            & 3.13\% / 19.12 $\pm$ 3.13
            & 0\% / $\star$
            & 34.38\% / 17.79 $\pm$ 5.00
            & 28.13\% / 18.55 $\pm$ 5.15 \\
    
            \bottomrule
        \end{tabularx}
    
        \caption{Simulated performance comparison over 32 trials of the bimanual Push-T and within-hand dynamic positioning experiments under unknown bounded uniform parameter distribution prior initialization. Results report success rate and completion time as mean $\pm$ standard deviation. $\star$ denotes insufficient completed trials to obtain completion-time statistics.}
        \label{tab:experiment_results}
    \end{table*}

    \subsection{Stein Variational DRO Formulation}
        Our goal in using Stein variational inference is to eliminate the restrictive Assumption \ref{assum:epsilon_pq} which is not necessarily guaranteed for a weak prior $p(\theta)$.
        We therefore require a more robust method  to optimize a control sequence $u^*$ that reasons about \emph{variations} of task sensitivities given that $\theta^*$ is unknown.
        To do so, we adopt the Stein-Variational approach \cite{liu2019steinvariationalgradientdescent} in the DRO formulation on the perturbed set of evolving estimators $\{\theta^i_t\}_{i=1}^N \forall t \in [0, T]$, where $T$ is the total number of time steps.

        
        We first define the evolving estimators $\hat{\theta}$ to follow the SVGD perturbation update,
        \begin{equation}\label{eq:svgd_params}
            \begin{split}
                \theta_{t+1}^i &\gets \theta_t^i + \alpha \phi^*(\theta_t^i)
            \end{split}
        \end{equation}
        where $\alpha > 0$ is the Stein evolve rate and  kernel $k(\theta,\hat{\theta})$ is an RBF (commonly used in Stein Variational inference \cite{lambert2021steinvariationalmodelpredictive,barcelos2021dualonlinesteinvariational}), and $\nabla_\theta \log p'(\theta)$ is the score of the log posterior in Eq.~\eqref{eq:log_posterior}.
        The estimators are then implemented to solve the following optimization problem that extends Eq. \ref{eq:softdro},
        \begin{equation}\label{eq:stein_dro_eq}
            \begin{split}
            \min_{\tau=\{(x_k,u_k,c_k) \forall k\in [t,t+t_h]\}} &\mathbb{E}_q \big[ \mathcal{L}(\tau, \theta) \big]  \approx \frac{1}{N} \sum_{i=1}^N \mathcal{L}(\tau, \theta^i) \\
            &= \mathcal{L} (\tau, \theta)\big|_{\theta=\bar\theta} +  \gamma \ \frac{1}{N} \sum_{i=1}^N \delta \mathcal{L}(\tau, \theta^i_t)
        \end{split}
        \end{equation}
        where $\bar{\theta}$ is the empirically averaged parameter value given the current parameter points, and $\mathcal{L}: \mathcal{T} \rightarrow \mathbb{R}$ where $\mathcal{T} \times \Theta \subset \mathcal{X} \times \mathcal{U} \times \mathcal{C_\theta}$ is the Lagrangian defined in Eq. \ref{eq:lagrangian}.
        Here, tuning parameter $\gamma \in (0, 1]$ is included in order to control the trade-off between optimizing over the nominal parameters versus parameter variations. 
        When $\gamma = 1$, the above formulation equates to the true Monte-Carlo approximation, but can be altered for numerical stability.
        When the auxiliary variable $\beta$ in Eq. \ref{eq:softdro} is sufficiently \emph{large}, the variance term becomes negligible, yielding the approximation in Eq. \ref{eq:log_expectation_dro_approx} where the objective reduces to the expectation term, which recovers Eq. \ref{eq:stein_dro_eq} under $\delta \mathcal{L}(\tau,\theta) \approx \mathcal{L}(\tau^*,\theta) - \mathcal{L}(\tau^*,\bar\theta)$.
        We approximate $\mathbb{E}_{\theta\sim p(\theta)}[\mathcal{L}(\tau, \theta)]$ empirically, as this term captures the gap in task-dependent optimality with respect to parameter uncertainty, which we aim to minimize.

        The core idea is to evaluate task-aware sensitivity gradients across parameter particles $\{\theta_t^i\}_{i=1}^N$, yielding a particle set that better represents variations that enlarge the optimality gap $\delta\mathcal{L}(\tau,\theta)$. 
        Alg.~\ref{alg:sv_dro} reflects the limited-feedback setting considered here: instead of replanning and updating particles after every control input, the controller executes the first $H$ controls, receives a new measured state $\tilde{x}$ from the measured dynamics $\tilde{f}$, and then resolves Eq.~\ref{eq:softdro}. 
        Thus, SVGD updates occur only at the lower feedback frequency, unlike estimation-driven Stein MPC methods that require continuous state-transition measurements for online system identification; here, particles evolve from task-sensitive optimality-gap gradients and must remain robust between intermittent measurements.
        
        Fig.~\ref{fig:intuition} illustrates this mechanism on bimanual Push-T. 
        Each translucent rollout is a predicted object trajectory induced by a different friction particle, with SVGD transporting broad prior samples toward parameter regions that expose large task optimality gaps. 
        The central principle of SV-DRO is therefore to \emph{reshape} uncertainty online into a task-sensitive particle set, rather than fully identifying physical parameters before control.
            
        \begin{algorithm}
        \caption{Stein Variational Distributionally Robust Control (SV-DRO)}\label{alg:sv_dro}
            \begin{algorithmic}[1]
            \Require planning horizon $t_h$, parameter prior $p(\theta)$, $N$ particles, step size $\alpha$, samples $\{\theta_0^i\}_{i=1}^N \sim p(\theta)$, parameterized surrogate model $f_\theta(x_k,u_k)$, time duration $T$, control execution window $H$.
            \State $t = 0$
            \While{$t< T$}
                \State $\tau^* \gets \arg \min_\tau \mathcal{L}(\tau,\theta)\big|_{\theta=\bar{\theta}} + \gamma \frac{1}{N}\sum_{i=1}^N \delta \mathcal{L}(\tau,\theta_t^i)$
                \For{$h=0,\ldots,H-1$}
                    \State Apply $u^*_{t+h}$ to the system
                    \State $\tilde{x}_{t+h+1} \gets \tilde{f}(\tilde{x}_{t+h},u^*_{t+h})$ 
                \EndFor
                \State $x_{t+H} \gets$ measured state after $H$ executed controls
                \State \text{Update SVI using the current task-aware posterior}
                \State $\{\theta_{t+H}^i\}_{i=1}^N \gets$ SVGD($\{\theta_t^i\}_{i=1}^N$) \Comment{Eq.~\ref{eq:svgd_params}}
                \State $t \gets t + H$
            \EndWhile
            \end{algorithmic}
        \end{algorithm}

    \begin{figure*}
        \centering
        \includegraphics[width=\linewidth]{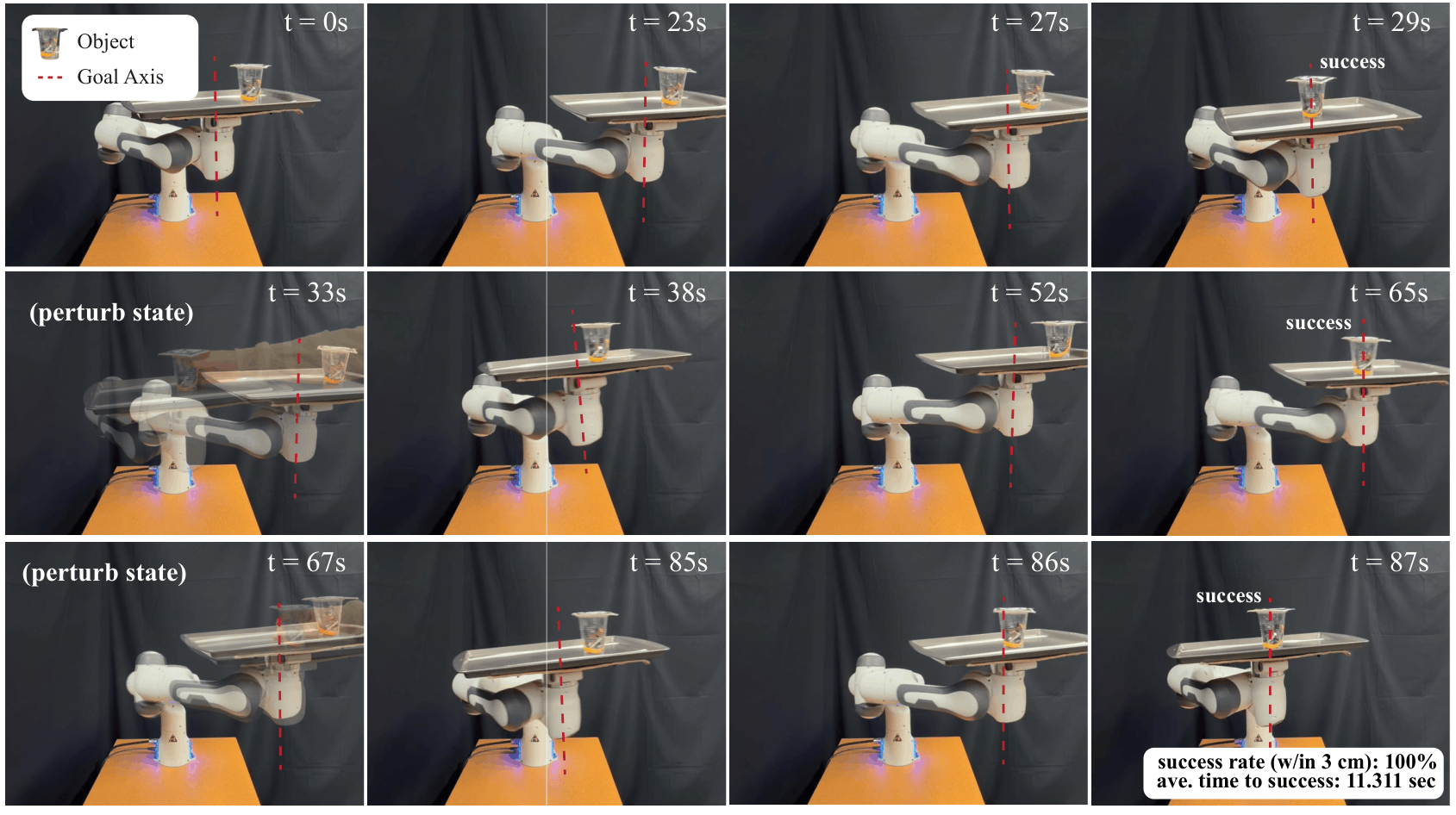}
        \caption{
        \textbf{Continual within-hand dynamic object positioning.}
        SV-DRO repeatedly positions a cup with unknown mass distribution and tray friction over a continuous 3.5-minute run using AprilTag pose feedback~\cite{5979561}. 
        After each successful placement within $1.5$ cm of the goal axis, the cup is randomly perturbed without stopping the experiment. 
        SV-DRO achieves the desired tolerance in approximately $100\%$ of trials, demonstrating reliable dynamic sliding under persistent physical uncertainty.
        }
        \label{fig:waiter_experiment}
        \vspace{-1.5em}
    \end{figure*}
        
\section{Results}
\label{sec:results}

    In this section, we demonstrate and analyze the efficacy of our proposed approach on a variety of contact-based control problems. To do so, we aim to address the following research questions:
    \begin{enumerate}[label=Q{\arabic*})]
        \item Can our proposed approach improve task success rate for manipulation tasks under broad parameter uncertainty?
        \item How does Stein variational inference impact the convergence of the DRO problem?
        \item Does the task-aware parameter inference via Stein Variational DRO benefit over existing uncertainty mitigating approaches?
        \item How does the choice of kernel affect the performance of Stein Variational DRO?
        \item What is the computation scaling for our Stein variational DRO approach with respect to particle parameter size and planning horizon?
    \end{enumerate}
    We specifically answer the above questions through implementations of our approach in the following contact-based control problems: 
    \begin{itemize}
        \item \textbf{Bimanual Push-T. } A bimanual manipulator is tasked with moving a T-shaped block with unknown mass and inertia to a goal pose through non-prehensile interactions.
        \item \textbf{Dynamic Within-Hand Positioning. } A robot manipulates an object with unknown mass distribution (mass and inertia) and friction coefficient to another goal state within its tray hand. A reference diagram is shown in Fig. \ref{fig:abstract}.
    \end{itemize}

    Implementation details for each control problem is detailed in Appendix \ref{app:implementation}. 
    We compare our approach to four following baselines: 
    \begin{enumerate}
        \item Ensemble Model Predictive Control (EMPPI) \cite{Abraham_2020,lowrey2019planonlinelearnoffline},
        \item Dual Stein Variational control (DuST-MPC) \cite{barcelos2021dualonlinesteinvariational}
        \item Conventional DRO \cite{kuhn2025distributionallyrobustoptimization,liu2025datadrivendistributionallyrobustoptimal,long2025sensorbaseddistributionallyrobustcontrol},
        \item and a naive MPC using the prior mean.
    \end{enumerate} 

    The key distinction between Stein-DRO and EMPPI lies in sample evolution, where Stein-DRO adaptively updates parameter particles via SVI shown in Eq. \ref{eq:svgd_params} using task-aware sensitivities, and EMPPI relies on static sampling from the prior distribution.
    DuSt-MPC uses SVGD primarily to perform online Bayesian system identification over dynamics parameters and controls from repeated state-feedback measurements.
    We address the above questions in this section with corresponding answers `A(-)'. We evaluate our experiments by initializing the algorithm's starting parameters from a bounded uniform distribution (more details in Appendix \ref{app:implementation}).
    For EMPPI and DRO, the sampled parameters are resampled after every new plan executed by the MPC controller as defined in Eq. \ref{eq:softdro}.
    Our method initially samples randomly from this parameter prior distribution, followed by subsequent SVGD evolution of the samples.
    \noindent\textbf{A1. Improved Performance Under Uncertainty}

        As shown in Table~\ref{tab:experiment_results}, SV-DRO achieves the highest success across bimanual Push-T and within-hand positioning while improving completion times. 
        Success is measured by final object distance to the goal, matching the task costs used by each controller.
        
        Fig.~\ref{fig:push-T-example-traj} shows SV-DRO coordinating two end effectors under uncertain physical parameters, while Fig.~\ref{fig:waiter_experiment} demonstrates reliable within-hand positioning under uncertain cup mass distribution and tray friction. 
        These tasks rely on intermittent contact and sliding, making them highly sensitive to latent dynamics and difficult for methods that require persistent state feedback for system identification, such as DuSt-MPC. 
        SV-DRO instead transports particles toward parameters that expose high optimality-gap sensitivity, allowing control to adapt to task-relevant uncertainty rather than merely fitting recent measurements. 
        Although our theoretical convergence guarantees (see Appendix \ref{app:proofs} for details) assumes stationary latent parameters, we also evaluate a discrete parameter shift in Fig.~\ref{fig:wine_serving_experiment}. 
        The robot first approaches a goal with a top-heavy wine glass of unknown mass distribution, then the glass is filled with water, abruptly changing its dynamics before the robot serves it to a second goal. 
        SV-DRO adapts to this sudden posterior shift and successfully completes the serving task, empirically demonstrating robustness beyond fixed-parameter settings.

    \begin{figure*}[ht!]
        \centering
        \includegraphics[width=\linewidth]{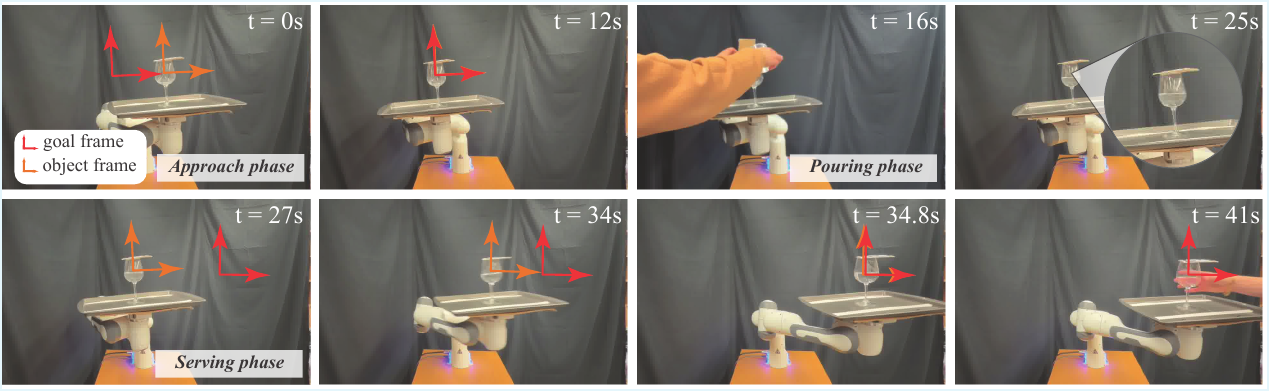}
        \caption{
        \textbf{Hardware demonstration of within-hand dynamic serving with a top-heavy object.}
        The robot moves a wine glass to a goal pose and then serves it to a second goal after the glass is abruptly filled with water, inducing a sudden shift in its mass distribution. 
        SV-DRO adapts to this posterior shift by using task-aware Stein variational gradients, enabling robust manipulation under changing physical uncertainty.
        }
        \label{fig:wine_serving_experiment}
        \vspace{-1.5em}
    \end{figure*}
    \noindent\textbf{A2. SVI Improves Task-Aware Convergence}

        We analyze the convergence properties of our approach in solving the DRO problem to characterize the impact of SVI on the within-hand dynamic object positioning task.
        Our focus is only on solving the trajectory optimization component in Alg. \ref{alg:sv_dro} until convergence, given an initialization according to a range of parameter priors drawn from a uniform bounded distribution. 
        We repeat this process over a set of 16 samples.
        
        The results of the DRO objective convergence are shown in Fig. \ref{fig:mpc_convergence} in comparison to the baseline methods. 
        We find that our approach has significantly less variance due to the deterministic evolution of the SVGD leading to consistent and reliable controllers despite the variability in parameter prior initialization.
        In contrast, the EMPPI approach experiences high variations of loss over solver iterations due to high stochasticity in the task-agnostic random sampling process itself, which is significantly more unreliable.
        The DRO approach converges to a worse-performing controller due to its risk averse nature.

    \begin{figure}
        \centering
        \includegraphics[width=0.9\linewidth]{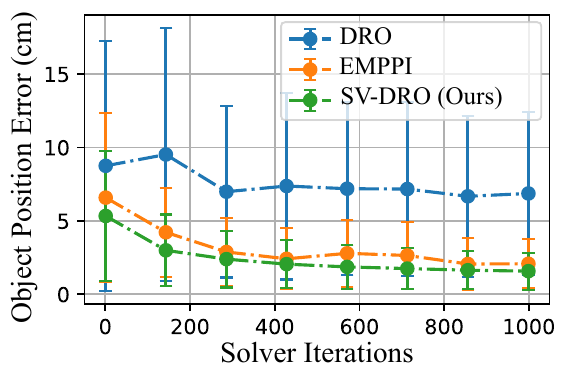}
        \caption{\textbf{Controller convergence based on prior distribution for Within-hand Positioning.} The number of parameter samples, planning horizon, and prior distributions, are fixed. 
        The results of this plot are evaluated over 5 trials per method with 5 random seeds. SV-DRO has improved reliability due to Stein inference. }
        \label{fig:mpc_convergence}
        \vspace{-1em}
    \end{figure}
        
    \noindent\textbf{A3. Task-Aware Inference Improves Manipulation}

        We show the effect on the parameter inference technique on the overall success of solving a manipulation task.
        Our proposed method diversifies the initially sampled parameters based on a task-aware posterior distribution via SVGD, which contributes to higher success rates indicated in Table \ref{tab:experiment_results}.
        However, what is the \emph{exact} relationship between the discrepancy of the evolved samples to the unknown true parameters and the success rate of the manipulation task at hand?
        
        We address this question by showing the initially sampled points from the parameter prior/posterior and their relationship to the respective performance of the within-hand dynamic object positioning experiment.
        Fig.~\ref{fig:param-discrepancy} relates parameter discrepancy (object distance to goal) to task error for the within-hand positioning experiment. 
        SV-DRO achieves consistently lower error across a wide range of discrepancies, showing that successful control does not require particles to converge to the true parameter $\theta^*$. 
        Instead, even with limited sensor feedback, SVGD transports particles toward task-relevant regions that expose high-sensitivity optimality-gap directions. 
        In contrast, DRO and EMPPI degrade as the prior mean moves away from the true parameter, since their uncertainty representations remain conservative or task-agnostic rather than task-adaptive.
        


    \begin{table}[t]
        \centering
        \begin{tabular}{lcc}
        \toprule
        Kernel & Success Rate $\%$ ($< 1$ cm) & Completion Time (s) \\
        \midrule
        RBF & $100.0\%$ & $12.93 \pm 4.78$ \\
        $k = 1$ & $100.0\%$ & $16.54 \pm 2.12$ \\
        IMQ & $100.0\%$ & $8.78 \pm 1.34$ \\
        \bottomrule
        \end{tabular}
        \caption{Effect of kernel choice on SV-DRO performance over 32 trials for within-hand object positioning problem. We compare the standard RBF kernel, a constant kernel, and the inverse multiquadric (IMQ) kernel. 
        The IMQ kernel achieves the fastest completion time, suggesting that its heavier-tailed distribution better preserves particle diversity over task-sensitive regions of the posterior and avoids premature mode collapse of Stein particles.}
        \label{tab:kernel_ablation}
        \vspace{-2em}
    \end{table}
    \noindent\textbf{A4. Choice of Kernel for SV-DRO}

        We evaluate how the SVGD kernel affects SV-DRO performance. The kernel determines particle interaction during posterior transport, balancing attraction toward task-sensitive regions with repulsion that preserves diversity. As shown in Table~\ref{tab:kernel_ablation}, all kernels achieve $100\%$ success, indicating that final task completion is not highly kernel-sensitive in this experiment. However, completion time varies substantially: the constant kernel, which removes distance-dependent repulsion, is slowest at $\sim 16.54$ s, suggesting that task-aware gradients alone can solve the task but less efficiently without geometry-aware particle diversity.
        
        The IMQ kernel achieves the fastest completion time at $\sim 8.78$ s, outperforming the RBF kernel at $\sim 12.93$ s. This suggests that stronger long-range particle interactions help maintain diversity over separated task-sensitive regions and avoid premature particle collapse. Thus, kernel choice primarily affects the efficiency, rather than feasibility, of SV-DRO; heavier-tailed kernels such as IMQ may be preferable in contact-rich tasks where uncertainty-sensitive modes are distributed across parameter space.

    \noindent\textbf{A5. Computational Scaling of SVI}

        Last, we investigate the computational complexity of the proposed SV-DRO approach for the within-hand dynamic positioning experiment. 
        We evaluate the computational time scaling with respect to the parameter sample size $N$ and the planning horizon $t_h$.
        All initial conditions, including random seeds and system state, are fixed.
        In evaluating over the parameter sample, we fix the planning horizon to $0.375$ seconds to be constant for all computations, and vice versa.
        The results shown in Fig. \ref{fig:comp_time_samples} demonstrate improved computational performance with GPU parallelization (NVIDIA GeForce RTX 3080) compared to the CPU (AMD Ryzen Threadripper 3960X 24-Core Processor).

    \begin{figure}
        \centering
        \includegraphics[width=\linewidth]{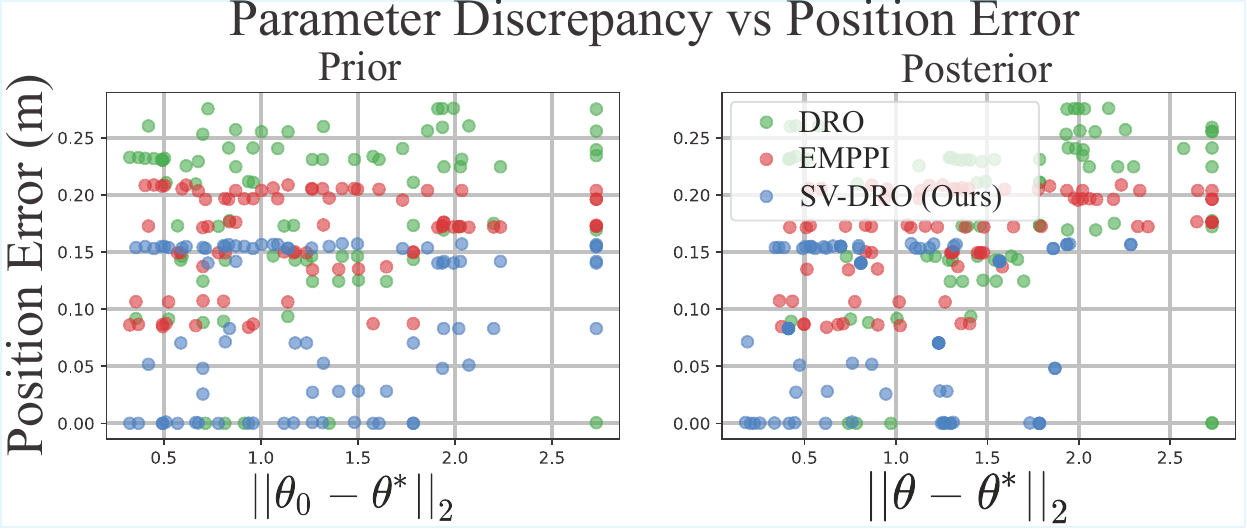}
        \caption{\textbf{Position Error over Parameter Samples. } Here, we demonstrate the relationship between the performance (measured as the distance to the goal state) and the initial parameter mean discrepancy for the within-hand object positioning experiment.
        Our approach adapts to task-sensitive parameter uncertainty leading to lower error despite the discrepancy. 
        We see weaker success rates for the samples from the EMPPI and DRO samples due to their task-agnostic posterior samples.}
        \label{fig:param-discrepancy}
        \vspace{-1.5em}
    \end{figure}

    \begin{figure}
        \centering
        \includegraphics[width=\linewidth]{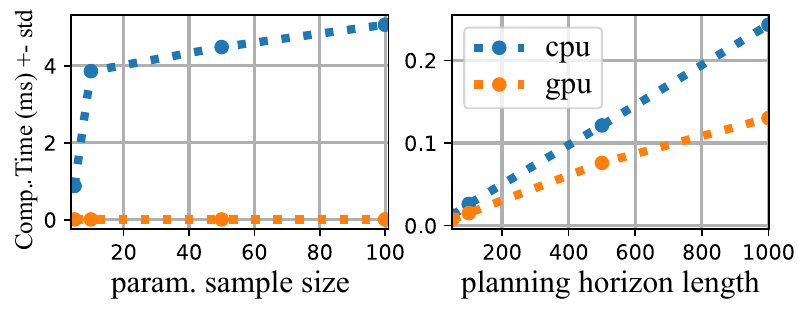}
        \caption{\textbf{Computational Analysis over Parameter Samples of SV-DRO. } Here, we demonstrate the computational scaling of the proposed SV-DRO gradient step as a function of the parameter sample size $N$ (left) and planning horizon $t_h$ (right). With the CPU (AMD Ryzen Threadripper 3960X 24-Core Processor), we observe linear scaling with an increase of sample size, and with the GPU (NVIDIA GeForce RTX 3080), we observe constant computational scaling (about 77 $\mu$s). For change in planning horizon length, we observe linear scaling with the CPU and sublinear scaling with GPU due to added parallelization.}
        \label{fig:comp_time_samples}
            \vspace{-2em}
    \end{figure}

\section{Limitations}
\label{sec:limitation}

\noindent\textbf{Model Error from Overlooked Parameters}
    A fundamental limitation of our proposed work is its model-based nature. 
    We require that the parameter uncertainty class be known in advance which has the risk of overlooked model parameters. 
    In particular, inaccurate modeling, like a contact-model that is unable to capture the relevant frictional forces may have an adverse effect on the overall performance.

\noindent\textbf{Model Complexity}
    A limitation of our approach, and generally model-based approaches, is the need to construct a reduced-order model to avoid mitigating computational complexity. 
    In more general open-world settings, the need to build such models could become a bottleneck, especially when one needs to accurately characterize uncertainty. 
    Model-free approaches excel in these scenarios but lack any form of interpretability. 
    An open challenge is bridging the gap between these classes of methods, especially when reliability and robust performance is necessary. 
    Future work will consider using the flexibility of Stein variational inference under less model-based restrictions.

\noindent\textbf{Sensitivity to Informative Priors}
    As shown in Eq.~\ref{eq:log_posterior}, the prior score $\nabla_\theta \log p(\theta)$ enters every SVGD update, with the task-aware optimality-gap gradient in Alg.~\ref{alg:sv_dro}. 
    Our bounded uniform prior has zero interior score, allowing task sensitivity to dominate particle transport. 
    More informative priors, such as Gaussians centered at nominal estimates, can persistently pull particles toward the prior mean and suppress exploration of task-critical uncertainty directions. 
    Future work will study structured priors to characterize this tradeoff.

    

\section{Conclusion} 
\label{sec:conclusion}

    In summary, we present an approach that enables robots to reliably complete manipulation tasks under broad parameter uncertainty.
    Our approach combines Stein variational inference and Distributionally Robust Optimization to accurately model uncertainty in parameters and synthesize control solutions that enhance task performance.
    We demonstrate accurate modeling of task-aware posterior via Stein variational inference and show significant improvements in contact-rich manipulation problems.
    Last, we showcase the reliability of our approach in real-world settings via a continually running within-hand manipulation experiment.

\section{Acknowledgements}
    This work is supported by the National Science Foundation under award NSF FRR 2238066. 
    Any opinions, findings, and conclusions or recommendations expressed in this material are those of the authors and do not necessarily reflect the views of the National Science Foundation.



\bibliographystyle{plainnat}
\bibliography{main}

    \appendices
    \section{Proofs}\label{app:proofs}
    
        \subsection*{SVGD Convergence in $\Theta$}\label{app:convergence_proof}

            Here we overview the theoretical guarantees of convergence of evolving parameters $\{\theta_t^i\}_{i=1}^N$ in $\Theta$.
            Similar to \cite{liu2019steinvariationalgradientdescent} using \cite[Corollary 6]{korba2021nonasymptoticanalysissteinvariational} shows that SVGD converges to the true posterior distribution with enough samples.
            
            \begin{theorem}[SVGD Convergence]
                Let $p'(\theta) = \exp \big(-V(\theta)\big)$ be the posterior in $\Theta$ where $V(\theta) = -\delta\mathcal{L} - \log p(\theta)$ is the $L$-Lipschitz smooth potential function with smooth prior $p$, $k$ be the universal RBF kernel over $\Theta$ with corresponding RKHS $\mathcal{H}$, and $\textrm{KSD}_{p'}(\mu_\theta)=\|\phi^*(\mu)\|_{\mathcal{H}^\Theta}^2$ is the Kernel Stein Discrepancy (KSD) over the posterior.
                Then, $\exists \zeta>0$ such that,
                \begin{equation}\label{eq:convergence}
                    \frac{1}{N} \sum_{i=1}^N\textrm{KSD}_{p'}(\mu_i) \leq \frac{D_{KL}(p\|p')}{Nc_\zeta}
                \end{equation}
                where $c_\zeta$ is a parametrized term defined in \cite[Corollary 6]{korba2021nonasymptoticanalysissteinvariational}.
            \end{theorem}
            \begin{proof}
                To prove the inequality in Eq. 20, we show that three conditions outlined in \cite[Section 5]{korba2021nonasymptoticanalysissteinvariational} are met.
                We first show that $\|k(\theta,\cdot)\|_{\mathcal{H}^\Theta} \leq \mathcal{B}$ since $k(\theta,\cdot)$ is a continuous RBF kernel over the bounded domain $\Theta$.
                Next, we show that the potential function is $L$-Lipschitz smooth,
                \begin{equation}
                    \|\nabla_\theta V(\theta_1) - \nabla_\theta V(\theta_2)\| \leq L\|\theta_2 - \theta_1\| \quad \forall \theta_1,\theta_2 \in \Theta
                \end{equation}
                since prior $p$ is smooth in $\Theta$ and the contact impulses embedded within the potential function $V$ are modeled to be $L$-smooth.
                Finally, we show that $\textrm{KSD}_{p'}(\mu) \leq \mathcal{M} \forall t \in [0,N]$ by observing the sufficient condition of probability measure $\mu$ over bounded $\Theta$,
                \begin{equation}
                    \sup_t \int_\Theta \|\theta^t\| d\mu(\theta^t) < \infty
                \end{equation}
                 thus concluding the proof.
            \end{proof}

            Note the proof alters a key assumption of the Hessian $\nabla^2V(\theta)$ from \cite[Corollary 6]{korba2021nonasymptoticanalysissteinvariational} by accounting for $L$-Lipschitz smooth potential functions due to the presence of contact.
            Contact-implicit interactions are especially challenging in that they pose non-smooth, non-differentiable constraints to the optimization problem.
            To provide guarantees that an interaction is $L$-smooth, this work models contact impulses using a softplus/softmax contact model \cite{leesoft}.

            Additionally, it is worth noting that the above remains valid for the primary setting of fixed physical parameter targets (e.g., mass and inertia), while the waiter-serving task shown in Fig. \ref{fig:wine_serving_experiment} with changing mass distribution lies outside its assumptions but still demonstrates strong empirical performance.
            Theoretical guarantees of convergence of moving parameter targets will be explored further in future work.

        \subsection*{Optimizing for Parameter-Sensitive Actions}

            It is possible to show that our approach in Eq. \ref{eq:stein_dro_eq} results in provably uncertainty-aware actions for optimal task completion.
            \begin{lemma}[Existence of Optima]
                Let $\tilde{\mathcal{L}} (\tau, \theta):\mathcal{T} \times \Theta\rightarrow \mathbb{R}$ be the objective function from Eq. \ref{eq:stein_dro_eq}.
                Then, $\exists \tau$ such that
                \begin{equation}
                    \tau^*= \arg \min_{\tau\in\mathcal{T}} \tilde{\mathcal{L}}(\tau, \theta)
                \end{equation}
            \end{lemma}
            \begin{proof}
                Because $\mathcal{T}$ and $\Theta$ are compact, and $\tilde{\mathcal{L}}_\theta$ is continuous and at least once-differentiable, then $\tau^*$ exists and is a minimizer to $\tilde{\mathcal{L}}_\theta$ by Weierstrass Extreme Value Theorem. 
            \end{proof}
            With the above lemma showing existence of minimizer $\tau^*$, we can now show resulting bounded uncertainty-aware actions with our proposed method.

        \begin{figure*}[t]
            \centering
            \includegraphics[width=\linewidth]{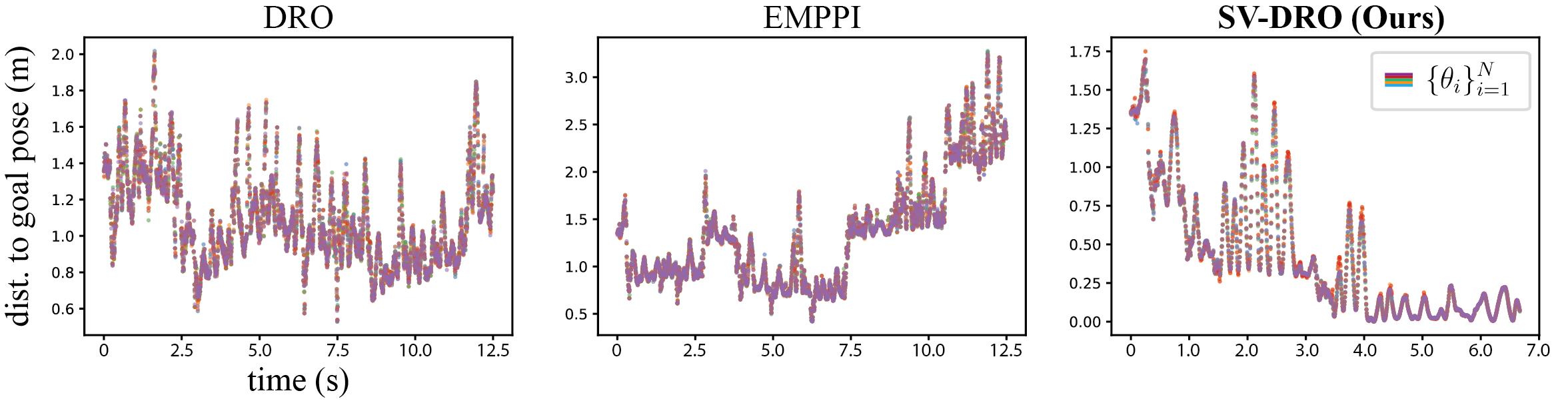}
            \caption{\textbf{Convergence of Parameters towards Task Completion. } We demonstrate a single within-hand dynamic positioning experiment for our approach and two comparative baselines.
            Five samples are drawn from an initial parameter prior uniform distribution and are subsequently evolved via SVGD (our approach) or through random uniform sampling in-the-loop.
            Our approach demonstrates improved convergence of parameters that collectively minimize the distance between the object and the goal pose.
            This outcome is due to the approach's task-aware nature that guides the initial samples towards a parameter posterior that explicitly reasons over task-relevant parameter uncertainty.
            This is in contrast to the other task agnostic approaches, which fail to converge their parameters towards optimizing for the task.}
            \label{fig:convergence}
        \end{figure*}
        \subsection*{Convergence of Optimality Gap}

            \begin{assumption}[Lipschitz-smooth parameterized dynamics]
                \label{assum:lip_smooth}
                Let $\mathcal X\subset\mathbb{R}^{n_x}$, $\mathcal U\subset\mathbb{R}^{n_u}$,
                and $\Theta\subset\mathbb{R}^{n_\theta}$ be compact. For every fixed
                $u\in\mathcal U$, the dynamics
                \begin{equation}
                x_{t+1} = f_\theta(x_t, u_t)
                \end{equation}
                are differentiable in $(\mathcal{X},\Theta)$, and there exist constants
                $L_x, L_\theta > 0$ such that for all $x,x'\in\mathcal X$, $u\in\mathcal U$,
                and $\theta,\theta'\in\Theta$,
                \begin{equation}
                \|f_\theta(x,u) - f_{\theta'}(x',u)\|
                \le L_x \|x - x'\| + L_\theta \|\theta - \theta'\|.
                \end{equation}
                Moreover, the stage cost $\ell(x,u,\theta)$ and terminal cost
                $\ell_T(x,\theta)$ are $L_x$ and $L_\theta$ smooth in $(\mathcal{X},\Theta)$.
            \end{assumption}
            
            \begin{theorem}[Convergence of the optimality gap]
                Let $q \in \mathcal P(\Theta)$ denote the target posterior distribution
                over parameters, and let $\{\mu_M\}_{M\ge 1}$ denote the empirical
                measures induced by Stein particles generated via SVGD. Suppose that
                \begin{equation}
                \mu_M \Rightarrow q \quad \text{as } M \to \infty,
                \end{equation}
                i.e., $\mu_M$ converges weakly to $q$.
                
                Fix an initial condition $x_0 \in \mathcal X$ and a control sequence
                $u_{0:T-1} \in \mathcal U^T$. Let the rollout be defined recursively by
                \begin{equation}
                x_{t+1}(\theta) = f_\theta(x_t(\theta), u_t),
                \end{equation}
                and define the trajectory-dependent Lagrangian
                \begin{equation}
                F(\theta) := L(x_{0:T}(\theta), u_{0:T-1}, \theta).
                \end{equation}
                
                Then $F$ is bounded and continuous on $\Theta$. Consequently, for every
                $\varepsilon > 0$, there exists $M_\varepsilon \in \mathbb Z_+$ such that
                for all $M \ge M_\varepsilon$,
                \begin{equation}
                \left|
                \mathbb E_{\theta \sim q}[F(\theta)]
                -
                \mathbb E_{\theta \sim \mu_M}[F(\theta)]
                \right|
                \le \varepsilon.
                \end{equation}
                
                Furthermore, defining the optimality gap
                \begin{equation}
                \delta L(\theta)
                =
                L(x_{0:T}(\theta), u_{0:T-1}, \theta)
                -
                L(x_{0:T}(\bar\theta), u_{0:T-1}, \bar\theta),
                \end{equation}
                for any fixed $\theta \in \Theta$, we have
                \begin{equation}
                \left|
                \mathbb E_{\theta \sim q}[\delta L(\theta)]
                -
                \mathbb E_{\theta \sim \mu_M}[\delta L(\theta)]
                \right|
                \le \varepsilon
                \end{equation}
                for all sufficiently large $M$.
            \end{theorem}
            
            \begin{proof}
                We first show that the rollout map $\theta \mapsto x_t(\theta)$ is
                Lipschitz continuous. For $t=0$, $x_0$ is fixed and independent of
                $\theta$. For $t \ge 0$, using Assumption 1,
                \begin{equation}
                \begin{aligned}
                \|x_{t+1}(\theta) - x_{t+1}(\theta')\|
                &=
                \|f_\theta(x_t(\theta), u_t)
                - f_{\theta'}(x_t(\theta'), u_t)\| \\
                &\le
                L_x \|x_t(\theta) - x_t(\theta')\|
                +
                L_\theta \|\theta - \theta'\|.
                \end{aligned}
                \end{equation}
                
                Applying this recursively yields
                \begin{equation}
                \|x_t(\theta) - x_t(\theta')\|
                \le
                L_\theta \sum_{s=0}^{t-1} L_x^s \|\theta - \theta'\|.
                \end{equation}
                Since $T$ is finite and $\Theta$ is compact, this implies that
                $\theta \mapsto x_{0:T}(\theta)$ is Lipschitz continuous.
                
                By Assumption \ref{assum:lip_smooth}, the cost terms are continuous and Lipschitz in their
                arguments. Therefore, 
                \begin{equation}
                F(\theta) = L(x_{0:T}(\theta), u_{0:T-1}, \theta)
                \end{equation}
                is continuous in $\theta$. Since $\Theta$ is compact, $F$ is bounded.
                
                By weak convergence $\mu_M \Rightarrow q$, and since $F$ is bounded and
                continuous, Portmanteau theorem \cite{billing} implies
                \begin{equation}
                \int_\Theta F(\theta)\, d\mu_M(\theta)
                \to
                \int_\Theta F(\theta)\, dq(\theta).
                \end{equation}
                Thus, for every $\varepsilon > 0$, there exists $M_\varepsilon$ such that
                for all $M \ge M_\varepsilon$,
                \begin{equation}
                \left|
                \mathbb E_{\theta \sim q}[F(\theta)]
                -
                \mathbb E_{\theta \sim \mu_M}[F(\theta)]
                \right|
                \le \varepsilon.
                \end{equation}
                
                Finally, note that
                \begin{equation}
                \delta L(\theta) = F(\theta) - F(\bar\theta),
                \end{equation}
                where $F(\bar\theta)$ is constant in $\theta$. Therefore,
                \begin{equation}
                \mathbb E_q[\delta L] - \mathbb E_{\mu_M}[\delta L]
                =
                \mathbb E_q[F] - \mathbb E_{\mu_M}[F],
                \end{equation}
                and the same convergence bound applies.
            \end{proof}

    \section{Implementation Details}\label{app:implementation}
    
    
        \begin{table*}[t]
            \centering
            \small
            \setlength{\tabcolsep}{12pt}
            \renewcommand{\arraystretch}{1.1}
            \begin{tabularx}{\textwidth}{@{} l >{\centering\arraybackslash}X >{\centering\arraybackslash}X @{}}
                \toprule
                \textbf{Term} & \textbf{Bimanual Push-T} & \textbf{Within-Hand Dynamic Positioning} \\
                \midrule
                \textbf{robot pose} &
                $[
                \underbrace{(x_\textrm{ee1}, y_\textrm{ee1})}_{\textrm{end effector 1}},
                \underbrace{(x_\textrm{ee2}, y_\textrm{ee2})}_{\textrm{end effector 2}},
                \underbrace{(x_\textrm{block}, y_\textrm{block}, \theta_\textrm{block})}_{\textrm{T block}}
                ]^\top$ &
                $[
                \underbrace{(x_\textrm{tray}, y_\textrm{tray}, \theta_\textrm{tray})}_{\textrm{tray}},
                \underbrace{(x_\textrm{object}, y_\textrm{object}, \theta_\textrm{object})}_{\textrm{object}}
                ]^\top$ \\
                \textbf{control (torque)} &
                $[(u_{x,\textrm{ee1}}, u_{y,\textrm{ee1}}), (u_{x,\textrm{ee2}}, u_{y,\textrm{ee2}})]^\top$ &
                $[u_{x,\textrm{tray}}, u_{y,\textrm{tray}}, u_{\theta,\textrm{tray}}]^\top$ \\
                \textbf{proprioceptive sensor noise} & 0.001 m & 0.01 m \\
                \midrule
                \multirow{4}{*}{\textbf{uniform sampling bounds of params.}} &
                mass: $( -1\textrm{e-}4,\,10.0)\,\textrm{kg}$ &
                mass: $( 5\textrm{e-}2,\,1.0)\,\textrm{kg}$ \\
                \addlinespace[2pt]
                &
                inertia: $( -1\textrm{e-}4,\,10.0)\,\textrm{kg m}^2$ &
                inertia: $( -1\textrm{e-}4,\,1.0)\,\textrm{kg m}^2$ \\
                \addlinespace[2pt]
                &
                --- &
                $x_\textrm{CoM}$: $( -1.2\textrm{e-}2,\,1.2\textrm{e-}2)\,\textrm{m} (\star)$ \\
                \addlinespace[2pt]
                &
                --- &
                $y_\textrm{CoM}$: $( -2.5\textrm{e-}2,\,2.5\textrm{e-}2)\,\textrm{m} (\star)$ \\
                \midrule
                \textbf{planning horizon} & 0.375 seconds & 0.375 seconds \\
                \textbf{total time} & 10 seconds & 12.5 seconds \\
                \textbf{spatial bounds} & -0.75 m $\times$ 0.75 m & -0.125 m $\times$ 0.125 m \\
                \textbf{control (torque) bounds [$\textrm{u}_\textrm{min}, \textrm{u}_\textrm{max}$]} & [(-1.0, -1.0), (1.0, 1.0)] N & [(-5.0, -0.2, -0.5), (5.0, 0.2, 0.5)] N \\
                \textbf{object dimensions} & $(T_\textrm{height},T_\textrm{width},T_\textrm{thickness}) = (0.5,0.6,0.2)$ m & $(\textrm{radius}_\textrm{obj.},\textrm{height}_\textrm{obj.})=(0.05,0.25)$ m \\
                \midrule
                \multirow{4}{*}{\textbf{MPC tuning weights}} 
                & diag($\mathbf{Q}$) = [5.0, 5.0, 2.0, 0.0, 0.0, 0.0] & diag($\mathbf{Q}$) = [15.0, 0.0, 5.0, 0.01, 0.01, 0.1, 0.0, 0.0, 5.0, 0.01, 0.01, 0.1]\\
                \addlinespace[2pt]
                & diag($\mathbf{R}$) = [0.0001, 0.0001] & diag($\mathbf{R}$) = [15.0, 20.0, 20.0] \\
                \addlinespace[2pt]
                & diag($\mathbf{N}$) = [0.1, 0.1] & diag($\mathbf{N}$) = [0.001, 0.001] \\
                \addlinespace[2pt]
                & diag($\mathbf{Q_f}$) = [5.0, 5.0, 2.0, 0.5, 0.5, 0.5] & diag($\mathbf{Q_f}$) = [15.0, 0.0, 5.0, 0.01, 0.01, 0.1, 0.0, 0.0, 5.0, 0.01, 0.01, 0.1] \\
                \midrule
                \textbf{$\theta$ Particle Count} & 5 & 5 \\
                \textbf{SVGD Step Size} & 0.0001 & 0.0001 \\
                \textbf{RBF Kernel bandwidth $h$} & 0.75 & 0.75 \\
                \midrule
                \textbf{num. control steps per planning cycle $H$} & 10 & 10 \\
                \bottomrule
            \end{tabularx}
            \caption{\textbf{Experiment Parameters.} Here we outline the experiment details for both the Bimanual push-T and the Within-Hand Dynamic Positioning experiments. We outline the state-control space of each experiment as well as uncertain model parameters. Each parameter of interest is sampled from a uniform distribution with indicated bounds. $\star$ indicates that the boundaries for the sample space is with respect to the object body coordinates, not world coordinates.}
            \label{tab:cost_values}
        \end{table*}
    
        \subsection*{Control Formulation}\label{app:control_formulation}
    
            Here, we write the optimal control formulation for the problem settings discussed in the paper.
            Let $\tau = \{(x_k,u_k,c_k) \forall k\in [t,t+t_h]\}$ be the trajectory for time $t \in [0,T]$, where $x \in \mathcal{X}=\mathbb{R}^{n_e}$ as the system state and $u \in \mathcal{U}=\mathbb{R}^{n_u}$ as the control input, and $c \in \mathcal{C}_\theta$ as the contact impulses.
            The task of finding an optimal control sequence $u^*$ is the solution to the following optimization for each planning cycle,
            \begin{equation}\label{eq:opt_formulation}
                \begin{split}
                    &\min_{\tau=\{(x_k,u_k,c_k) \forall k\in [t,t+t_h]\}} m(x_{t_h},u_{t_h},c_{t_h}) + \sum_{k=t}^{t+t_h} \ell(x_k,u_k,c_k) \\
                    \textrm{s.t.} &
                    \begin{cases}
                        x_0 \in \mathcal{X} & \textrm{(init. state)} \\
                        x_{k+1} = f_\theta(x_k,u_k,c_k) & \textrm{(dynamical model)}
                    \end{cases}
                \end{split}
            \end{equation}
            where the terminal $m(\cdot)$ and running cost $\ell(\cdot)$ are written as,
            \begin{equation}
                \begin{split}
                    \ell(x_k,u_k,c_k) &= (x_{k,\textrm{obj.}} - x^*_\textrm{obj.})^\top \mathbf{Q} (x_{k,\textrm{obj.}} - x^*_\textrm{obj.}) \\
                    &+ u_k^\top \mathbf{R} u_k + c_k^\top \mathbf{N} c_k \\
                    m(x_k,u_k,c_k) &= (x_{t_h,\textrm{obj.}} - x^*_\textrm{obj.})^\top \mathbf{Q_f} (x_{t_h,\textrm{obj.}} - x^*_\textrm{obj.})
                \end{split}
            \end{equation}
            where `obj.' is the object of interest to optimize (e.g. T-block), and $\mathbf{Q}, \mathbf{Q_f} \in \mathbb{R}^{n_e \times n_e}$, $\mathbf{R} \in \mathbb{R}^{n_u \times n_u}$, and $\mathbf{N} \in \mathbb{R}^{n_c \times n_c}$ are weight matrices along spatial, controllable, and contact dimensions, respectively.
            Additionally, the bimanual push-T has an additional cost function to penalize the end-effector to block signed distance function, $\phi^n$ defined as,
            \begin{equation}
                \ell_\phi (x_k,u_k) = {\phi_k^n}^\top \phi_k^n.
            \end{equation}

        \subsection*{Experiment Details}\label{app:experiments}
    
            This section outlines the experimental setup for the bimanual push-T and the within-hand dynamic positioning problem.
            Additional experiment parameters are detailed in Table \ref{tab:cost_values}.
            Note that the above table indicates experiment details for the simulation experiments.
            For the hardware implementation of the within-hand dynamic waiter, proprioception of the object and tray states are measured using April Tags \cite{5979561}.
            The dimensions of the unknown object are roughly $5$ cm in diameter and $11$ cm in height.
    
        \subsection*{Convergence of Inferred Parameters }\label{app:convergence}
    
            Here, we empirically show that the Stein-inferred parameters converge towards near optimal task completion for a single within-hand dynamic positioning experiment. 
            Results are shown in Fig. \ref{fig:convergence}.
            Five parameter samples are initialized from a uniform prior distribution, and the samples are evolved online via SVGD for our approach, or are randomly sampled from the prior for the EMPPI and DRO approaches.
            Our SV-DRO approach not only demonstrates improved performance over time, but we also observe that the initial samples are later evolved and collapse towards values that minimize the distance of the object to the goal state. 
            Our approach benefits from such outcomes due to the evolved samples being computed by a task-aware posterior distribution that explicitly considers parameter uncertainty with respect to the manipulation task.
            In contrast, the baseline approaches not only fail to achieve comparative optimal performance due to their higher distance to goal, but also fail to converge their random samples toward minimizing the task objective due to their stochastic, task-agnostic nature.

        \begin{table}[t]
            \centering
            \label{tab:compute_parity}
            \begin{tabular}{lc}
                \toprule
                Method & Wall-Clock / Step (s) \\
                \midrule
                SV-DRO (Ours) & $0.130 \pm 0.008$ \\
                DuST-MPC & $0.151 \pm 0.009$ \\
                EMPPI & $0.087 \pm 0.007$ \\
                MPC & $0.068 \pm 0.011$ \\
                DRO & $0.123 \pm 0.008$ \\
                \bottomrule
            \end{tabular}
            \caption{Compute parity across baselines for a single MPC loop. All methods are evaluated under the same compute setting with an MPC planning horizon of $0.375$ s on an NVIDIA GeForce RTX 3080 GPU, and particle-based methods use $N=5$ parameter particles. SV-DRO has a wall-clock time of $0.130 \pm 0.008$ s per step, which is comparable to DRO and remains below the MPC planning horizon despite the additional SVGD posterior update.}
            \label{tab:comp_scaling}
        \end{table}
        \subsection*{Additional Computational Results}\label{app:compute}

            We report additional compute-parity results across all baselines to clarify that the performance gains in Table~\ref{tab:experiment_results} are not obtained from substantially larger computation budgets. 
            All methods are evaluated on the same hardware configuration using an NVIDIA GeForce RTX 3080 GPU, with an MPC planning horizon of $0.375$ s and $N=5$ parameter particles for particle-based methods. 
            
            These results are consistent with the computational scaling analysis in Table~\ref{tab:comp_scaling}, where the proposed SV-DRO gradient step benefits from GPU parallelization over parameter particles and planning horizon. Importantly, instead of static prior sampling or purely estimation-driven parameter updates, SV-DRO transports particles toward task-sensitive regions of the optimality-gap posterior and embeds them in the robust MPC objective. This supports the central claim that task-aware particle evolution improves robustness under limited feedback while maintaining practical online computation. 
\end{document}